\documentclass[a4paper,12pt]{article}

\usepackage[marginparwidth=2cm]{geometry}
\usepackage[utf8]{inputenc}
\usepackage[unicode=true,hypertexnames=false]{hyperref}
\usepackage{amsmath,amssymb,amsthm,amsfonts}
\usepackage{cite}
\usepackage[table]{xcolor}
\usepackage{tikz}
\usetikzlibrary{arrows.meta,calc,positioning}
\usetikzlibrary{external}
\usepackage{marginnote}

\usepackage{algorithm}
\usepackage{algorithmicx}
\usepackage{algpseudocode}

\usepackage{array,hhline} 

\makeatletter
\let\pgfmathModX=\pgfmathMod@
\usepackage{pgfplots}%
\let\pgfmathMod@=\pgfmathModX
\makeatother
\pgfplotsset{compat=1.14}

\newcommand{\theop}{\mathcal{X}}

\newtheorem{theorem}{Theorem}
\newtheorem{lemma}{Lemma}

\newtheorem{definition}{Definition}

\pgfplotstableread{
  x y dev
    100 436.94 109.40138495605541
    200 1063.25 234.42098774416402
    300 1660.59 358.5757768976842
    400 2304.34 530.8648788495615
    500 3034.71 561.1439487441
    600 3890.26 872.9204277202551
    700 4494.0 831.5155811024599
    800 5356.03 1017.4169953982743
    900 5855.16 1105.1481798355996
    1000 6806.78 1204.311426383214
    1100 7699.28 1447.3986632939286
    1200 8246.03 1473.4365482494222
    1300 8866.29 1524.24725780606
    1400 9871.22 1688.5451067039155
    1500 10898.8 1871.2692985988917
    1600 11514.46 2185.058561589542
    1700 12072.24 1814.3643855731857
    1800 13051.53 2400.4924303529515
    1900 14060.43 2549.9852495315117
    2000 14993.51 2488.696845700171
}{\ArityOne}
\pgfplotstableread{
  x y dev
    100 197.93 13.433152867424534
    200 395.09 19.496746685360897
    300 598.18 24.22553123537983
    400 792.51 26.113000361482555
    500 1001.75 30.955939688857807
    600 1204.97 30.8268068440768
    700 1401.62 38.43840146351961
    800 1600.18 39.61771361806779
    900 1794.02 47.407108532307035
    1000 1998.15 43.81766052332522
    1100 2202.78 44.63535311246845
    1200 2392.37 45.39916232039904
    1300 2602.11 52.360231458463176
    1400 2798.66 47.60956953293226
    1500 2999.42 51.64705003945315
    1600 3203.2 56.349199845537505
    1700 3392.34 54.2512337522048
    1800 3590.62 58.501796570876365
    1900 3805.69 63.78510243724071
    2000 3998.03 70.20121944896124
}{\ArityTwo}
\pgfplotstableread{
  x y dev
    100 111.58 7.326691206841192
    200 223.46 10.311727167245898
    300 336.72 13.003558176611714
    400 450.9 13.991700281512495
    500 563.2 19.053354114966297
    600 673.65 18.77908518549401
    700 789.62 21.908339816959057
    800 902.45 18.33381542065884
    900 1010.4 24.89127874529063
    1000 1124.6 27.832108483301493
    1100 1236.67 23.787381002319165
    1200 1346.96 26.516592861257593
    1300 1459.82 26.508767369961724
    1400 1568.48 30.436381738526894
    1500 1684.32 32.929640327844645
    1600 1795.96 32.05170065927104
    1700 1908.56 30.0761054861184
    1800 2025.16 26.1115400781441
    1900 2134.54 32.50560790778332
    2000 2252.25 33.666629162895404
}{\CustomArityThree}
\pgfplotstableread{
  x y dev
    100 85.5 4.945153734166962
    200 172.1 7.607147091300844
    300 258.37 9.26234366135322
    400 342.92 9.270001579985266
    500 427.5 12.036560467126307
    600 512.1 11.970080546213058
    700 597.9 13.137578148043865
    800 683.87 13.571103436524547
    900 767.3 14.014782959937033
    1000 854.46 15.893216900122633
    1100 937.88 16.008381643010726
    1200 1026.03 18.79418295992973
    1300 1111.96 19.441141422199987
    1400 1197.11 17.967363791501572
    1500 1284.14 18.84761861012492
    1600 1363.0 21.70509412813294
    1700 1456.34 20.46639028537107
    1800 1539.28 23.722758965375277
    1900 1623.08 23.825832682831923
    2000 1707.61 21.1325279456514
}{\CustomArityFour}
\pgfplotstableread{
  x y dev
    100 128.93 7.480958994323351
    200 257.83 9.972745688407809
    300 388.99 13.651872998370454
    400 515.38 16.81256144144081
    500 649.45 15.496741267085474
    600 771.89 18.26012656643492
    700 902.3 21.032202341969214
    800 1035.47 19.157747595406224
    900 1157.3 24.053181817704612
    1000 1288.46 21.479809641200777
    1100 1419.18 25.070156107865127
    1200 1549.03 27.547004823527995
    1300 1675.93 25.895772050684453
    1400 1806.79 26.459534570679235
    1500 1940.69 28.928227101437162
    1600 2063.61 30.16382708885702
    1700 2200.58 31.82755936847082
    1800 2325.44 31.693073519160436
    1900 2455.3 33.01744258950888
    2000 2585.71 35.70736703547407
}{\GenericArityThreePure}
\pgfplotstableread{
  x y dev
    100 128.03 8.0734694581022
    200 257.63 10.593170088564266
    300 385.55 13.462536870512146
    400 515.67 15.715947636111865
    500 647.14 16.189546446439262
    600 771.34 19.35120386201052
    700 904.21 19.00945803934174
    800 1036.47 20.050166880083363
    900 1164.8 22.512847061806355
    1000 1290.54 20.7332358957079
    1100 1422.72 22.120459929096125
    1200 1543.71 24.42026604972586
    1300 1678.18 23.677709758208014
    1400 1811.22 27.450443870489025
    1500 1938.18 27.38700569262876
    1600 2068.11 28.27458972721663
    1700 2196.37 32.0468895731416
    1800 2328.73 32.456919388038294
    1900 2451.86 36.9184249529974
    2000 2583.2 35.57038832023801
}{\GenericArityThreeHack}
\pgfplotstableread{
  x y dev
    100 95.2 5.126678094324657
    200 192.23 7.581016962243138
    300 291.15 10.261603434673596
    400 383.3 11.106918400872328
    500 480.04 12.660221645816954
    600 573.45 12.619029241185673
    700 673.21 12.57968138858785
    800 767.7 14.45962683062077
    900 863.4 17.003861712308563
    1000 960.83 17.622331447812623
    1100 1055.64 19.851224423473507
    1200 1154.61 18.237432368454154
    1300 1249.62 22.557571686776047
    1400 1341.28 21.07757676831572
    1500 1442.43 19.521681807827438
    1600 1532.97 24.022907838355064
    1700 1635.36 21.030771106484266
    1800 1726.41 24.221659478218825
    1900 1821.32 22.522213501565542
    2000 1916.63 27.225083131464487
}{\GenericArityFourPure}
\pgfplotstableread{
  x y dev
    100 93.32 4.95816844645625
    200 188.19 5.923237921504373
    300 279.67 6.78687830830325
    400 373.67 8.702786039452246
    500 467.22 9.921489787121665
    600 561.56 10.242829470197933
    700 651.29 10.673572575063682
    800 746.3 12.17589936113654
    900 841.67 14.140032003164109
    1000 935.51 13.747723411165175
    1100 1025.18 15.017417497120737
    1200 1122.03 15.043677150956878
    1300 1214.4 13.851302677314722
    1400 1305.76 15.418813118388778
    1500 1399.11 14.498968963691906
    1600 1495.44 14.684782200728357
    1700 1584.74 16.007965189084764
    1800 1679.65 18.127703560121354
    1900 1774.58 15.38671865964642
    2000 1868.98 18.192561472461573
}{\GenericArityFourHack}
\pgfplotstableread{
  x y dev
    100 71.57 4.434153489868968
    200 141.32 5.880768176304384
    300 207.9 7.568542018303676
    400 278.99 7.158374342846523
    500 348.81 9.838796633324582
    600 415.62 9.814213563948357
    700 487.0 11.03255055357426
    800 556.83 11.511132073612982
    900 625.66 12.141605571693793
    1000 695.09 12.540918884381895
    1100 766.66 13.705576539599793
    1200 830.47 15.66405845471496
    1300 900.38 15.192887756573413
    1400 972.48 16.898078308051787
    1500 1040.4 15.333992080711186
    1600 1108.13 16.099222015207747
    1700 1178.63 18.670808523175722
    1800 1248.99 20.08491820730021
    1900 1319.89 18.251273651419236
    2000 1388.12 19.19505407676937
}{\GenericArityFivePure}
\pgfplotstableread{
  x y dev
    100 66.41 3.035463787305063
    200 132.3 4.806351017248955
    300 196.62 5.153757087912862
    400 262.62 6.203257534586485
    500 328.29 7.641321742983863
    600 390.93 7.2254530458076935
    700 458.3 8.295672475344105
    800 526.21 8.401893725929781
    900 585.5 8.731112370845036
    1000 654.31 9.163713022767363
    1100 720.36 10.288288919392231
    1200 782.97 9.367313684112728
    1300 850.38 11.000257113615874
    1400 917.55 11.70804350983899
    1500 975.89 12.519798462249842
    1600 1044.13 12.293264148888337
    1700 1109.93 12.210489777484796
    1800 1171.52 12.445014416961673
    1900 1240.85 13.505984457329124
    2000 1306.35 12.874197042726559
}{\GenericArityFiveHack}
\pgfplotstableread{
  x y dev
    100 53.68 3.074430225606236
    200 103.25 3.9628960943412905
    300 152.97 5.054161201875682
    400 200.05 5.49080958502386
    500 253.97 7.565812259864444
    600 306.17 7.293355215744691
    700 355.95 7.027183869183524
    800 404.74 7.442194745252405
    900 454.52 8.589963605713049
    1000 507.41 8.78025217649404
    1100 556.16 10.114216417368862
    1200 606.61 9.020621717022395
    1300 654.66 10.901969055768285
    1400 706.58 10.39345170690846
    1500 758.51 10.3587127237379
    1600 807.07 10.749916607616408
    1700 855.74 11.331211922665807
    1800 906.79 11.116731406835248
    1900 958.44 11.444481710246267
    2000 1010.76 12.632425808932414
}{\GenericAritySixPure}
\pgfplotstableread{
  x y dev
    100 50.12 2.731337038924713
    200 96.75 3.3736201818392555
    300 142.79 3.745151748081991
    400 189.05 3.9090028178939797
    500 239.92 5.250646905791033
    600 287.32 4.846908834765236
    700 332.45 5.770658978877983
    800 379.8 5.872380137972038
    900 426.87 6.84408614436691
    1000 477.0 7.015135152805056
    1100 521.91 6.479329994043917
    1200 569.95 7.36545993132812
    1300 614.85 7.671770369742603
    1400 664.84 7.976885800279902
    1500 714.76 8.056768282103887
    1600 758.23 7.913490590967642
    1700 804.69 7.394640428763719
    1800 853.15 9.581521567869869
    1900 904.07 8.907215325144264
    2000 951.86 9.44630759994228
}{\GenericAritySixHack}

\pgfplotstableread{
x y dev
2 2.6866 1.4530628269824817
3 3.7151 1.900761236801722
4 4.7269 2.0691410112538637
5 5.5443 2.0928152086483203
6 6.2225 2.0781784409821253
7 6.8715 2.028200954025743
8 7.4118 1.9847454953892614
9 7.8818 1.9640810883087676
10 8.2976 1.9262412311359836
11 8.7572 1.87440643936965
12 9.0693 1.8580211877891617
13 9.5019 1.8250286197608516
14 9.8415 1.82491391059355
15 10.18 1.8005344185600245
16 10.4958 1.766492119835464
17 10.8287 1.7179788866994323
18 11.1219 1.7511559167318285
19 11.4624 1.7133825632300783
20 11.7321 1.6857680063926954
21 12.0472 1.7184491458133324
22 12.355 1.678111023043644
23 12.6379 1.6533169495573834
24 12.9301 1.6434975200497988
25 13.1905 1.6127212521026528
26 13.4811 1.6212049902964165
27 13.7737 1.6083988947709413
28 14.0448 1.5904860655134487
29 14.3001 1.5860931748027494
30 14.5698 1.589773788332711
31 14.8579 1.5518854419197263
32 15.1332 1.5517404904993464
}{\UnrestrictedPure}
\pgfplotstableread{
x y dev
2 2.4504 0.9964131291349018
3 2.9984 1.0604291070077054
4 3.5426 1.0996845777844453
5 4.0366 1.1731999317285813
6 4.5476 1.2308069420396412
7 5.001 1.3052085492118721
8 5.4971 1.3569730017847803
9 5.9137 1.3698320897086163
10 6.3605 1.3579853321801114
11 6.7845 1.3701268097324064
12 7.1556 1.3909644561280585
13 7.5817 1.372484419346252
14 7.9582 1.360013868963201
15 8.3057 1.366760517951646
16 8.659 1.3323274783816905
17 9.0176 1.330363569375946
18 9.347 1.349063748419469
19 9.6835 1.314344133037523
20 9.9831 1.3227960420778702
21 10.3303 1.3011422701306674
22 10.6609 1.3100697761119164
23 10.967 1.3106039708000468
24 11.2728 1.3075745237322138
25 11.5749 1.2956688874039621
26 11.8745 1.2909362500295485
27 12.2032 1.3006455811991602
28 12.5099 1.3000273062901093
29 12.787 1.314459501231613
30 13.1087 1.289205381252936
31 13.3948 1.2924008627354076
32 13.7009 1.2915130624965012
}{\UnrestrictedHack}

\pgfplotscreateplotcyclelist{myplotcycle}{%
black\\
red!80!black\\
violet!80!black\\
gray!80!black\\
orange\\
brown!80!black\\
cyan!80!black\\%
green!70!black\\
green\\
blue!60!black\\
teal\\
magenta!70!black\\
yellow!90!black\\
}

\begin{document}

\title{Better Fixed-Arity Unbiased Black-Box Algorithms\footnote{
An extended abstract of this work will appear in proceedings of Genetic and Evolutionary
Computation Conference 2018. The paper is accompanied with the source code,
which is available on GitHub: \url{https://github.com/mbuzdalov/unbiased-bbc}.
}}

\author{Nina Bulanova \and Maxim Buzdalov}

\maketitle

\begin{abstract}
In their GECCO'12 paper, Doerr and Doerr proved that the $k$-ary unbiased black-box complexity
of \textsc{OneMax} on $n$ bits is $O(n/k)$ for $2\le k\le O(\log n)$.
We propose an alternative strategy for achieving this unbiased black-box complexity 
when $3\le k\le\log_2 n$.
While it is based on the same idea of block-wise optimization, 
it uses $k$-ary unbiased operators in a different way.

For each block of size $2^{k-1}-1$ we set up, in $O(k)$ queries, a virtual coordinate system,
which enables us to use an arbitrary unrestricted algorithm to optimize this block.
This is possible because this coordinate system introduces a bijection between unrestricted 
queries and a subset of $k$-ary unbiased operators.
We note that this technique does not depend on \textsc{OneMax} being solved and can be used
in more general contexts.

This together constitutes an algorithm which is conceptually simpler than the one by Doerr 
and Doerr, and at the same time achieves better constant factors in the asymptotic notation.
Our algorithm works in $(2+o(1))\cdot n/(k-1)$, where $o(1)$ relates to $k$. 
Our experimental evaluation of this algorithm shows its efficiency already for $3\le k\le6$.
\end{abstract}

\section{Introduction}

Unbiased black-box complexity of a problem is one of the measures for how complex this problem is
for solving it by evolutionary algorithms and other randomized heuristics. Often, complexities
of rather simple problems are studied, such as the famous \textsc{OneMax} problem, defined on bit strings of length $n$ as follows:
\begin{equation*}
\textsc{OneMax}_{z}: \{0,1\}^n \to \mathbb{R}; x \mapsto |\{i \in [1..n] \mid x_i = z_i\}|.
\end{equation*}

The notion of unbiased black-box complexity was introduced in~\cite{unbiased-bbc} for pseudo-Boolean problems
(see also the journal version~\cite{unbiased-bbc-algorithmica}) partially as a response
to unrealistically low black-box complexities of various NP-hard problems~\cite{droste-jansen-wegener}.
Since evolutionary algorithms and other randomized search heuristics are designed as
general-purpose solvers, they shall not prefer one instance of a problem over another one.
This means that unbiased black-box algorithms are a better model of randomized search heuristics,
since the unbiased model does not allow certain ways of being ``too fast''.
Unfortunately, the ways were found to perform most of the work without making queries
in the unbiased model too~\cite{too-fast-unbiased-bbc,unbiased-partition-is-polynomial}.
In fact, it was shown that, with a proper notion of unbiasedness for the given type of individuals,
the unbiased black-box complexity coincides to the unrestricted one~\cite{generic-unbiased-algorithms}.
Several alternative restricted models of black-box algorithms were subsequently introduced
as a reaction, namely ranking-based algorithms~\cite{ranking-based-complexity},
limited-memory algorithms~\cite{mastermind-constant-memory},
and elitist algorithms~\cite{elitist-complexity}. 

One of possible restrictions to the unbiased black-box search model is the use of unbiased operators
with restricted arity. The original paper~\cite{unbiased-bbc} studied mostly unary unbiased black-box complexity,
e.g.~the class of algorithms allowing only unbiased operators taking one individual and producing another
one, or mutation-based algorithms. This model appeared to be quite restrictive, e.g.~the unary unbiased
black-box complexity of \textsc{OneMax} was proven to be $\Theta(n \log n)$~\cite{unbiased-bbc,doerr-doerr-yang-optimal-parameter-choices-gecco}. 
Together with the rather old question, whether crossover is useful in evolutionary algorithms
(which was previously positively, but only in artificial settings~\cite{jansen-crossover}),
this inspired a number of works on higher-arity unbiased algorithms,
since many crossovers are binary unbiased operators.

The theorem that for $k \ge 2$ the
$k$-ary unbiased black-box complexity of \textsc{OneMax} is $O(n / \log k)$, which was proven 
in~\cite{doerr-johannsen-faster-blackbox}, was the first signal that higher arities are useful.
Among others, an elegant crossover-based algorithm with the expected running time of $2n - O(1)$
was presented, which works for all linear functions. Several particular properties of this algorithm
inspired the researchers to look deeper for faster general-purpose algorithms that use crossover.
The first reported progress of algorithms using crossover on simple problems like \textsc{OneMax}
was made in~\cite{sudholt-crossover-speeds-up,sudholt-crossover-speeds-up-evco}, where an algorithm was presented
with the same $O(n \log n)$ asymptotic but with a better constant factor.
An algorithm, called the $(1 + (\lambda,\lambda))$ genetic algorithm, was
presented in~\cite{learning-from-black-box-thcs} along with the proof of the $O(n \sqrt{\log n})$ runtime
on \textsc{OneMax}, which was faster than any evolutionary algorithm before,
and improved performance on some other problems.
With the use of self-adaptation for its parameter $\lambda$, the $O(n)$ bound was proven for the runtime
on \textsc{OneMax}~\cite{doerr-doerr-lambda-lambda-self-adjustment}, and similar improvements were shown
later on a more realistic problem, MAX-SAT~\cite{buzdalovD-gecco17-3cnf}.
Experiments show that the constant in $O(n)$ is rather small.

After such a great success already for binary unbiased operators, what can we expect from higher arity
algorithms? Some of the existing algorithms already feature as much as quaternary ($k = 4$) operators,
the notable example of which is differential evolution~\cite{differential-evolution}.
The $O(n / \log k)$ result on \textsc{OneMax} does not look very inspiring in this aspect,
since it suggests that the arity should grow exponentially to get noticeable runtime improvements.
However, soon a better result appeared in~\cite{doerr2012reducing},
which shows that the $k$-ary unbiased black-box complexity of \textsc{OneMax} is $O(n / k)$ 
for $2 \le k \le O(\log n)$.
While there are still no matching lower bounds for arities greater than one,
this result suggests that higher arities generally pay off.

The result presented in~\cite{doerr2012reducing} is rather complicated and non-trivial to use
for two reasons. First, it has to use, as a building block, a derandomized unrestricted algorithm to solve
\textsc{OneMax}, proposed in~\cite{erdos-renyi-two-problems}, which runs in
$(1+\delta) \log_2(9) \cdot n / \log_2 n$ for $\delta$ decreasing with $n$.
The corresponding sequence is only proven to exist, but no way to construct it has ever been proposed.
Second, this algorithm needs $4\ell$ bits to optimize a piece of $\ell$ bits and encodes the fitness
values of several queried strings into some of these bits to make the final choice. This
not only complicates the algorithm but also increases the arity needed to
find the optimum within the given number of fitness queries.

In this paper we aim to improve this situation. We propose an algorithm which, like the one 
from~\cite{doerr2012reducing}, optimizes \textsc{OneMax} by blocks of $\ell$ bits but does it
in a simpler and more explicit manner and performs fewer queries. For $k$-ary operators,
$k \ge 3$, we set $\ell = 2^{k-1}-1$ and initialize a ``virtual coordinate system'' in $k$ queries
very similar to ``storage initialization'' in~\cite{doerr2012reducing}. However, subsequently
we use the opportunities offered by this coordinate system in a different manner. We notice that
it introduces a bijection between bit strings of length $\ell$ and 
$2^\ell$ different $k$-ary unbiased operators. We use this fact to optimize these $\ell$ bits by simulating an 
arbitrary algorithm for solving \textsc{OneMax} of length $\ell$, regardless of unbiasedness or arity of this algorithm.

Note that this algorithm operates not on the individuals of the enclosing $k$-ary unbiased algorithm,
and even not on their parts, but in fact on the subset of $k$-ary unbiased operators, where each operator
is unambiguously defined by a bit string of length $\ell$.
This subset is isomorphic to a set of bit strings of length $\ell$, and the above-mentioned bijection
effectively introduces an instance of the \textsc{OneMax} problem defined on this subset. In simple words,
each such $k$-ary unbiased operator has an associated ``fitness value''. It is computed
by applying this operator to a predefined sequence of already queried strings, computing the fitness of the resulting
string and translating it back using a simple linear transformation which is detailed later.

While the $k$-ary unbiased algorithm cannot access the particular bits of the bit string,
the definition of a $k$-ary unbiased algorithm does not prevent it from 
setting the parameters of the $k$-ary unbiased operators, which it uses, in an arbitrary way depending on the
preceding query history. Our algorithm is thus able to run an arbitrary optimizer on its $k$-ary unbiased operators,
while the fitness values of these operators are computed in a completely unbiased way with the use of at most $k$-ary 
unbiased operators.

In particular, we can use the random sampling algorithm from~\cite{erdos-renyi-two-problems}
which works in expected $(2 + o(1)) \cdot \ell / \log \ell$ time.
Thus, the overall number of queries needed by the proposed algorithm to solve the \textsc{OneMax} problem
sums up from expected $(2 + o(1)) \cdot n / (k - 1)$ queries from solving each block
and $O(n \cdot k / 2^k)$ additional work coming from the initialization of coordinate systems and 
the aggregation of answers found for each block.
Note that the second addend is negligibly small in $k$ compared to the first one,
so it hides entirely in $o(1)$ inside the first addend.

The rest of the paper is structured as follows.
Section~\ref{operators} introduces the general form of unbiased operators, which is used
afterwards to describe the used operators in the concise way without the need to prove their unbiasedness.
Section~\ref{existing} gives a short overview of the existing $O(n/k)$ algorithm 
from~\cite{doerr2012reducing} so that one can clearly see the differences between this algorithm
and the proposed one.
Section~\ref{system} describes the ``virtual coordinate system'' that we use to optimize blocks.
In Section~\ref{algorithms} the proposed algorithm is described in the general form,
along with a faster custom version for $k = 3$. 
Section~\ref{experiments} presents the empirical study of the proposed algorithms for $3 \le k \le 6$,
which shows that the improvements from higher arities are clearly seen already
for these values of $k$. This empirical study also suggests that a simple modification to
the $(2 + o(1)) \cdot n / \log n$ random sampling algorithm from~\cite{erdos-renyi-two-problems}
might reduce its running time, which seems to be interesting on its own.
Section~\ref{conclusion} concludes.


\section{Unbiased Operators: General Form}\label{operators}

In this section we give the formal definition of $k$-ary unbiased operators,
and then we show how to define every possible $k$-ary unbiased operator in a concise form,
which additionally enables to implement this operator in a constructive way.
We also introduce a convenient notation for $k$-ary unbiased operators, which we will use in subsequent sections.
In what follows we consider only operators which take and produce bit strings of length $n$.

Note that most of this section seems to belong to a common sense in theory of evolutionary computation.
However, no literature seems to cover this fact. In particular, the paper~\cite{doerr2012reducing}
still uses the explicit distribution notation, and the authors still have to prove that certain distributions are unbiased.

\begin{definition}
A $k$-ary operator $\theop$, which produces a search point $y$ from the given $k$ search points $x^{(1)}, \ldots, x^{(k)}$
with probability $P_{\theop}(y \mid x^{(1)}, \ldots, x^{(k)})$,
is \emph{unbiased} if the following relations hold for all search points $x^{(1)}, \ldots, x^{(k)}, y, z$ and all permutations $\pi$ over $[1..n]$:
\begin{align}
    P_{\theop}(y \mid x^{(1)}, \ldots, x^{(k)}) &= P_{\theop}(y \oplus z \mid x^{(1)} \oplus z, \ldots, x^{(k)} \oplus z), \label{flip-inv} \\
    P_{\theop}(y \mid x^{(1)}, \ldots, x^{(k)}) &= P_{\theop}(\pi(y) \mid \pi(x^{(1)}), \ldots, \pi(x^{(k)})), \label{swap-inv}
\end{align}
where $a \oplus b$ is the bitwise exclusive-or operation applied to $a$ and $b$, and $\pi(a)$ is an application of permutation $\pi$ to $a$.
\end{definition}

This definition is equivalent to Definition~1 in~\cite{doerr2012reducing}.
In simple words, \eqref{flip-inv} declares that $\theop$ is invariant under flipping the $i$-th bits, for any $i$, in every argument and in the result simultaneously,
and \eqref{swap-inv} declares that $\theop$ is invariant under permuting bits in the same way in arguments and the result.

It is clear that the only $0$-ary unbiased operator is the operator that produces a bit string with every bit set with probability $0.5$.
The standard bit mutation is an example of the $1$-ary, or \emph{unary}, unbiased operator, while the uniform crossover operator,
as well as the crossover operator from the $(1+(\lambda,\lambda))$ algorithm~\cite{learning-from-black-box-thcs}\footnote{%
This crossover is called ``biased'' in that paper, which can cause confusion. In fact, the ``biasedness'' of this crossover
is about choosing the bits from the parent with much greater probability than from the child. However, in the terms of the current paper,
this crossover is unbiased.}
are $2$-ary, or \emph{binary}, unbiased operators.

We are going to characterize all unbiased operators for $k \ge 1$. To do this, we note that the set of bit indices,
$\{1, 2, \ldots, n\}$, can be split into $2^{k-1}$ disjoint sets according to whether the $i$-th bit
is same in the first argument $x^{(1)}$ and each of other arguments, $x^{(2)}$ to $x^{(k)}$.
More formally, we introduce sets $S_0, S_1, \ldots, S_{2^{k-1}-1}$, such that $S_i \cap S_j = \emptyset$ if $i \ne j$
and $\bigcup S_i = \{1, 2, \ldots, n\}$. It is easier to construct these sets by defining the index of the
set, to which a particular bit index $i$ belongs, as follows:
\begin{equation*}
    i \in S_j \leftrightarrow j = \sum_{t=2}^{k} 2^{t-2} \cdot \left[x^{(1)}_i \ne x^{(t)}_i\right],
\end{equation*}
where $[P]$ is the Iverson bracket, which yields 1 if $P$ is true and 0 otherwise.
For instance, $S_0$ is the set of bit indices at which bits are equal in all arguments,
while $S_{2^{k-1}-1}$ contains indices at which $x^{(1)}$ differs from all other arguments.

Note that the sets $S_i$ do not change if all arguments $x^{(i)}$ are replaced with $x^{(i)} \oplus z$ for an arbitrary bit string $z$
same to all arguments. In the same time, it holds for an arbitrary permutation $\pi$ that whenever $i \in S_j$,
then $\pi(i) \in S^{*}_j$, where the set $S^{*}_j$ is built from $\pi(x^{(1)}), \pi(x^{(2)}), \ldots, \pi(x^{(k)})$.
We may define $\pi(S_j) = \{ \pi(i) \mid i \in S_j \}$, which is the same as $S^{*}_j$. Note also that $|S_j| = |\pi(S_j)|$ for an arbitrary $\pi$.

The following lemma states that a certain kind of $k$-ary operator is unbiased. We will show later that
\emph{every} $k$-ary unbiased operator can be described as an operator of this kind.

\begin{lemma}
Consider a $k$-ary operator $\theop$ on bit strings of length $n$ which, given $k$ arguments $x^{(1)}, \ldots, x^{(k)}$:
\begin{enumerate}
    \item Computes sets of bit indices $S_0, \ldots, S_{2^{k-1}-1}$ as above.
    \item Computes sizes of these sets $n_j = |S_j|$.
    \item Based only on the values of $n_j$, chooses values $d_0, \ldots, d_{2^{k-1}-1}$ such that $0 \le d_j \le n_j$.
          These values can heavily depend on the $n_j$ values and need not be random. \label{genform-3}
    \item Chooses uniformly at random subsets $F_j \subseteq S_j$, such that $|F_j| = d_j$, for all $0 \le j < 2^{k-1}$, and computes $F = \bigcup F_j$. \label{genform-4}
    \item Returns a bit string $y$ which differs from $x^{(1)}$ exactly at indices from the set $F$.
\end{enumerate}
The operator $\theop$ is unbiased.
\label{genform}
\end{lemma}

\begin{proof}
    To prove unbiasedness of $\theop$, one needs to show that \eqref{flip-inv} and \eqref{swap-inv} hold.
    The first holds trivially since, as shown above, applying the $(z \oplus \cdot)$ operation, for an arbitrary bit string $z$, to all arguments
    does not change the sets $S_i$, which, in turn, leaves intact the distribution of the indices which are going to be flipped.

    To prove the second property, consider $P_{\theop}(y \mid x^{(1)}, \ldots, x^{(k)})$. This is exactly the probability
    that $F = \{ i \mid x^{(1)}_i \ne y_i \}$ is chosen at step~\ref{genform-4}, given the arguments $x^{(1)}, \ldots, x^{(k)}$.
    This event is equivalent to the event that $F_j \subseteq S_j$ for all $j$ are chosen. The probability of this event
    depends only on $n_j$ and $d_j$, since every possible subset of $S_j$ of size $d_j$ is chosen with the same probability.
    We also note that $F_j = \{ i \mid y_i \ne x^{(1)}_i \} \cap S_j$.
    
    Now consider $P_{\theop}(\pi(y) \mid \pi(x^{(1)}), \ldots, \pi(x^{(k)}))$.
    As shown above, the sizes of the index sets do not change, that is, $|\pi(S_j)| = |S_j| = n_j$.
    One can see that $|\pi(F_j)| = |F_j| = d_j$ as well.
    This means that this new probability is exactly the same as the probability in the previous case,
    which proves the second property and the entire lemma.
\end{proof}

We prove now that every $k$-ary unbiased operator can be represented in the way described in Lemma~\ref{genform}.
We do not need this fact in the subsequent sections, since both the algorithm from~\cite{doerr2012reducing} and the proposed
algorithm use only operators which can be explicitly represented as Lemma~\ref{genform} requires.
However, we feel that this fact deserves a proof.

\begin{lemma}
Assume $\theop$ is a $k$-ary unbiased operator which operates on bit strings of length $n$.
Then for any $y, x^{(1)}, \ldots, x^{(k)}, \tilde{y}, \tilde{x}^{(1)}, \ldots, \tilde{x}^{(k)}$ such that, for all $j \in \{0,\ldots, 2^{k-1}-1\}$:
\begin{enumerate}
\item The index sets $S_j$ generated from $x^{(1)}, \ldots, x^{(k)}$
      and the index sets $\tilde{S}_j$ generated from $\tilde{x}^{(1)}, \ldots, \tilde{x}^{(k)}$ 
      satisfy $|S_j| = |\tilde{S}_j|$.\label{every-is-genform-1}
\item $|\{ i \mid y_i \ne x^{(1)} \} \cap S_j| = |\{ i \mid \tilde{y}_i \ne \tilde{x}^{(1)} \} \cap \tilde{S}_j|$;\label{every-is-genform-2}
\end{enumerate}
it holds that $P_{\theop}(y \mid x^{(1)}, \ldots, x^{(k)}) = P_{\theop}(\tilde{y} \mid \tilde{x}^{(1)}, \ldots, \tilde{x}^{(k)})$.
\label{every-is-genform}
\end{lemma}

\begin{proof}
Let $U = \{ 1, \ldots, n \}$, $F = \{ i \mid y_i \ne x^{(1)}_i \}$ and $\tilde{F} = \{ i \mid \tilde{y}_i \ne \tilde{x}^{(1)}_i \}$.
We define $S_j^0 = S_j \cap (U \setminus F)$, $S_j^1 = S_j \cap F$, $\tilde{S}_j^0 = \tilde{S}_j \cap (U \setminus \tilde{F})$,
$\tilde{S}_j^1 = \tilde{S}_j \cap \tilde{F}$.

By part~\ref{every-is-genform-2} of the lemma statement, $|S_j^0| = |\tilde{S}_j^0|$ and $|S_j^1| = |\tilde{S}_j^1|$.
This means that there exists a bijection between each such pair. Since $\bigcup S_j = U$ and $\bigcup \tilde{S}_j = U$,
there exists a permutation $\pi$ such that $\pi(S_j^0) = \tilde{S}_j^0$ and $\pi(S_j^1) = \tilde{S}_j^1$ for all $j$ simultaneously.

Define $z = \pi(y) \oplus \tilde{y}$. Since bits in every $S_j^0$ coincide in $x^{(1)}$ and $y$,
bits in $\pi(S_j^0) = \tilde{S}_j^0$ will also coincide in $\pi(x^{(1)})$ and $\pi(y)$, so they will also coincide
in $\pi(x^{(1)}) \oplus z$ and $\pi(y) \oplus z = \tilde{y}$. This means that $\tilde{x}^{(1)}$ coincides with $\pi(x^{(1)}) \oplus z$
in bits $\tilde{S}_j^0$ for every $j$. The same logic shows that $\tilde{x}^{(1)}$ coincides with $\pi(x^{(1)}) \oplus z$
in bits $\tilde{S}_j^1$ for every $j$, so $\tilde{x}^{(1)} = \pi(x^{(1)}) \oplus z$.

By construction of sets $S_j$ and $\tilde{S}_j$, it also holds that $\tilde{x}^{(i)} = \pi(x^{(i)}) \oplus z$ for all $1 \le i \le k$. 
This means that, in the statement to be proven, the arguments of $P_{\theop}(y \mid x^{(1)}, \ldots, x^{(k)})$ are translated into
the arguments of $P_{\theop}(\tilde{y} \mid \tilde{x}^{(1)}, \ldots, \tilde{x}^{(k)})$ by applying the permutation $\pi$ first,
and then $(z \oplus \cdot)$. Since the operator $\theop$ is unbiased, these two probabilities are equal.
\end{proof}

To represent an arbitrary $k$-ary unbiased operator in the form which Lemma~\ref{genform} specifies,
it is enough to iterate over all possible values of $n_j$, to generate an arbitrary set of arguments such that $|S_j| = n_j$,
and to measure, for every possible combination of values $d_j$, the probability of sampling an arbitrarily chosen bit string
which matches these values of $d_j$. By Lemma~\ref{every-is-genform}, the probabilities for all the remaining cases will be
automatically determined.

The only part of the unbiased operator, which is done at will, is the stage~\ref{genform-3} of the form required by Lemma~\ref{genform}.
This, together with Lemma~\ref{every-is-genform}, proves the following theorem.

\begin{theorem}
Every $k$-ary unbiased operator defined on its arguments $x^{(1)}, \ldots, x^{(k)}$ is uniquely described by the following, possibly non-deterministic, mapping:
\begin{equation*}
\langle n_0, n_1, \ldots, n_{2^{k-1}-1} \rangle \to \langle d_0, d_1, \ldots, d_{2^{k-1}-1} \rangle,
\end{equation*}
where $n_j$ is the size of the subset of bit indices $S_j$ such that:
\begin{equation*}
    S_j = \left\{ i \:\;\middle|\:\; j = \sum_{t=2}^{k} 2^{t-2} \cdot \left[x^{(1)}_i \ne x^{(t)}_i\right] \right\},
\end{equation*}
and $d_j$ indicates how many bits, chosen uniformly at random from $S_j$, to flip in $x^{(1)}$ to achieve the result.
\label{unbiased-repr}
\end{theorem}

We give a few examples of the notation introduced in this theorem.
\begin{itemize}
    \item Single bit mutation: $\langle n \rangle \to \langle 1 \rangle$.
    \item Unary inversion operator: $\langle n \rangle \to \langle n \rangle$.
    \item Standard bit mutation: $\langle n \rangle \to \langle \text{Bin}(n, 1/n) \rangle$.
    \item Uniform crossover: $\langle n_0, n_1 \rangle \to \langle \text{Bin}(n_0, 1/2), \text{Bin}(n_1, 1/2) \rangle$.
    \item Crossover from~\cite{doerr-johannsen-faster-blackbox} for the binary case: $\langle n_0, n_1 \rangle \to \langle 0, 1 \rangle$.
\end{itemize}

The result presented in this section essentially means that a $k$-ary unbiased operator can distinguish at most $2^{k-1}$ groups of bits,
determined by which arguments have a certain bit different to that of the first argument,
and it is able to choose different numbers of bits to flip in each of these groups.

\section{The Existing Algorithm: Overview}\label{existing}

In this section we cover the $k$-ary unbiased algorithm for solving \textsc{OneMax} on bit strings of length $n$,
which was proposed in~\cite{doerr2012reducing}.

This algorithm uses operators with arity at most $k \ge 7$. For convenience, denote $\kappa = k - 7$.
The algorithm maintains two strings $x$ and $y$, which coincide in precisely those bits that are guessed right
and differ in all other bits. Initially $x$ is random, and $y$ is the inverse of $x$.
Unless $x = y$, the algorithm wants to choose at most $\ell = 2^{\kappa}$ bit indices among those which differ in $x$ and $y$ at random
and optimizes bits at these indices.

To optimize these bits, the algorithm uses the ``derandomized random sampling'' technique: before it starts its work, it precomputes
a string-distinguishing sequence $r_1, \ldots, r_t$ for $t = (1+o(1)) \log_2(9) \cdot \ell / \log_2 \ell$.
However, since the algorithm is unbiased, it cannot directly place the strings from this sequence into the $\ell$ chosen bits.
What is more, the algorithm requires some additional space to store the fitness values, received from the ``bigger'' \textsc{OneMax} of length $n$
and recoded to be understood by a solver of the ``smaller'' \textsc{OneMax} of length $\ell$. Since a fitness in the latter cannot be greater than
$\ell$, it cannot occupy more than $\kappa + 1$ bits. All fitness values corresponding to the queries from the sequence $r_1, \ldots, r_t$ 
cannot occupy more than $4\ell$ bits in total.

To accomodate them, the algorithm initializes the so-called \emph{storage} of length $4\ell$, of which $\ell'$ bits
($\ell' \le \ell$, where the inequality can be strict in the last run)
are those which it intends to optimize. First, a string $y_0$ is queried by a binary operator
$\langle n_0, n_1 \rangle \to \langle 4\ell - \ell' - q, \ell' + q\rangle$, where
$\max(0, 4\ell - \ell' - n_0) \le q \le n_1 - \ell'$ and
\begin{equation*}
P(q = q_0) = \frac{\binom{n_0}{4\ell - \ell' - q_0} \binom{n_1 - \ell'}{q_0}}{\binom{n_0 + n_1 - \ell'}{4\ell - \ell'}}
\end{equation*}
in order to distribute $4\ell - \ell'$ bits of storage uniformly at random, applied to $x$ and $y$.
In this way, the bits in which it differs from $x$ define the bits allocated for the storage,
and the $\ell$ bits to optimize are among them.

In order to address the individual bits of the storage, the strings $y_i$ are generated,
the $i$-th one by an $(i+2)$-ary unbiased operator
$\langle 0, n_1, \ldots, n_j, \ldots \rangle \to \langle 0, 0, \ldots, [j \bmod 2 = 0] \cdot n_j / 2, \ldots \rangle$
applied to $x, y, y_0, \ldots, y_{i-1}$. The only exception is the string $y_1$, where it is additionally ensured
that $y_1$ coincides with $x$ in the $\ell'$ bits which are to be optimized.
Our construction, described in detail in Section~\ref{system}, is very similar to this one,
so we will discuss its working principles later. Now we only state that
the strings $x, y, y_0, \ldots, y_{\kappa+2}$ introduce a bijection
between all bit strings of length $4\ell$ and all bit strings which are different from $x$ only in those $4\ell$ bits
in which $x$ and $y_0$ differ. What is more, this bijection can be expressed as a composition of some permutation and an
application of the operation $\oplus$ with some bit string.
Note that this bijection, which is Hamming-preserving,
is not explicitly mentioned in~\cite{doerr2012reducing},
where one of its directions is called the \texttt{write} procedure.
The bijection $\sigma$, which was mentioned in this paper, is a bijection between the indices of these strings.

The algorithm proceeds with morphing the strings $r_1$ to $r_t$ by the above mentioned bijection
(using $(\kappa + 5)$-ary operators applied each to $x$, $y$, $y_i$), and querying them. To morph these fitnesses back,
another string, $y_B$, must be queried beforehand by the $4$-ary operator
$\langle n_0, n_1, n_2, \ldots, n_7 \rangle \to \langle 0, n_1, 0, \ldots, 0 \rangle$
applied to strings $x$, $y_0$, $y_1$, $y_2$.%
These morphed fitnesses, each of $\kappa+1$ bits, are then written to the storage
using $(\kappa + 6)$-ary operators with the help of the same bijection.
Finally, a $(\kappa + 7)$-ary operator chooses the bits consistent with the stored fitness values and the
pre-computed values $r_1, \ldots, r_t$ and writes them to the $\ell'$ bits which they should occupy.

Since every $\ell$ bits, chosen in this way, are optimized in $O(\ell / \log \ell)$ queries, the entire
running time of the algorithm is $O(n / \log \ell) = O(n / \kappa) = O(n / k)$.

In our opinion, all the major complications in this algorithm arise solely from the fact that the authors of~\cite{doerr2012reducing}
accidentally missed the fact that the system of strings $x, y, y_0, \ldots, y^{\kappa+2}$ actually introduced
a Hamming-preserving bijection between the $2^{4\ell}$ bit strings which are the candidates for an optimum in $4\ell$ bits
and the same number of $(\kappa+5)$-ary unbiased operators which take all uninteresting bits from the first argument
but can change every single bit among the $4\ell$ interesting bits, compared to its value in the first argument, in arbitrary way independently
of all other $4\ell - 1$ bits. With this weapon in hand, we can run an arbitrary unrestricted algorithm for solving \textsc{OneMax}
through the prism of this bijection, since the unbiased setting does not restrict us which $k$-ary unbiased operators
to apply. We develop this idea in next two sections.

\section{The Virtual Coordinate System}\label{system}

In this section, we assume that we have two strings $x$ and $y_0$, such that they differ in $\ell$ bits.
We need to establish a way to set each of these bits in either the same value as in $x$, or in a different value than in $x$,
independently of other bits, in a single query, while leaving all other $n - \ell$ bits in the same state as in $x$.
Since for $\ell = 1$ this problem is trivial, we consider $\ell \ge 2$.
We show a way to do this with the help of $k$-ary unbiased operators such that $2^{k-1}-1 \ge \ell$.

Recall that a $k$-ary unbiased operator can distinguish $2^{k-1}$ groups of bits. One of these groups has to be dedicated
to ``unrelated'' bits, i.e. the $n-\ell$ bits outside of the interesting region. For simplicity, we allocate the $0$-th group
for this purpose, that is, the group of bits which are equal throughout all arguments. Our aim is to design auxiliary queries
in such a way that each of the remaining $2^{k-1}-1$ groups consists of at most a single bit,
which enables subsequent fine manipulations with them.

We illustrate this technique on $\ell = 2^{k-1}-1$. It will also work with only minor changes for any $2^{k-2} \le \ell < 2^{k-1}$,
whereas smaller $\ell$ will only require a smaller $k$. It can also be modified to work slightly better
in the case of $2^{k-1} \ge n$, in which we do not need to care about the unrelated bits since there are none of them.

To create the next $k - 1$ bit strings $y_1, \ldots, y_{k-1}$, we use a set of very similar unbiased operators of arity $2, \ldots, k$ correspondingly,
which can be described by a single phrase: in each bit group except the $0$-th one, flip half of the bits \emph{rounded up}.
The formal notation for these operators is as follows:
\begin{equation*}
\langle n_0, n_1, \ldots, n_j, \ldots \rangle \to \langle 0, \lceil n_1/2 \rceil, \ldots, \lceil n_j/2 \rceil, \ldots \rangle.
\end{equation*}

Each time such an operator is applied, a new string is created which, if used together with the previous ones,
divides every bit group with the size greater than one (except the $0$-th one) into two parts of (almost) equal size.
When $y_{k-1}$ is created, every group of $2^{k-1}-1$ such groups contains a single element (at most one element if $\ell < 2^{k-1}-1$).
The algorithm for constructing strings $y_1, \ldots, y_{k-1}$
from $x$ and $y_0$ is outlined as Algorithm~\ref{algo-vcs}.
Figure~\ref{example-vcs} illustrates it for $k=5$ and $\ell = 15$. 

\begin{algorithm}[!t]
\caption{Construction of the virtual coordinate system}\label{algo-vcs}
\begin{algorithmic}
\Function{BuildCoordinates}{$x, y_0; k$}
    \State{-- $x$ and $y_0$ are strings which differ in $\ell$ bits, $2^{k-2} \le \ell < 2^{k-1}$}
    \State{-- $k$ is the maximum allowed arity}
    \State{$\textsc{FlipUpperHalf} := \langle n_0, \ldots, n_j, \ldots \rangle \to \langle 0, \ldots, \lceil n_j/2 \rceil, \ldots \rangle$}
    \For{$i \in [1; k-1]$}
        \State{$y_i \gets \Call{FlipUpperHalf}{x, y_0, \ldots, y_{i-1}}$; $\Call{Query}{y_i}$}
    \EndFor
    \State\Return{$(y_1, \ldots, y_{k-1})$}
\EndFunction
\end{algorithmic}
\end{algorithm}

\begin{figure}[!t]
\centering
\setlength{\tabcolsep}{1pt}
\begin{tabular}{|c|*{15}{>{\centering}p{1em}|}l}
\cline{1-16}
0\ldots0 & 1 & 0 & 1 & 0 & 1 & 0 & 1 & 0 & 1 & 0 & 1 & 0 & \cellcolor{gray!50!white} 1 & 0 & 1 & ~$y_4$ \\
\hhline{*{16}{=}~}
0\ldots0 & 0 & 1 & 1 & 0 & 0 & 1 & 1 & 0 & 0 & 1 & 1 & 0 & \cellcolor{gray!50!white} 0 & 1 & 1 & ~$y_3$ \\
\hhline{*{16}{=}~}
0\ldots0 & 0 & 0 & 0 & 1 & 1 & 1 & 1 & 0 & 0 & 0 & 0 & 1 & \cellcolor{gray!50!white} 1 & 1 & 1 & ~$y_2$ \\
\hhline{*{16}{=}~}
0\ldots0 & 0 & 0 & 0 & 0 & 0 & 0 & 0 & 1 & 1 & 1 & 1 & 1 & \cellcolor{gray!50!white} 1 & 1 & 1 & ~$y_1$ \\
\hhline{*{16}{=}~}
0\ldots0 & 1 & 1 & 1 & 1 & 1 & 1 & 1 & 1 & 1 & 1 & 1 & 1 & 1 & 1 & 1 & ~$y_0$ \\
\hhline{*{16}{=}~}
0\ldots0 & 0 & 0 & 0 & 0 & 0 & 0 & 0 & 0 & 0 & 0 & 0 & 0 & 0 & 0 & 0 & ~$x$ \\
\cline{1-16}
\multicolumn{1}{c}{\rule{0pt}{2.5ex}0} &
\multicolumn{1}{c}{8} &
\multicolumn{1}{c}{4} &
\multicolumn{1}{c}{12} &
\multicolumn{1}{c}{2} &
\multicolumn{1}{c}{10} &
\multicolumn{1}{c}{6} &
\multicolumn{1}{c}{14} &
\multicolumn{1}{c}{1} &
\multicolumn{1}{c}{9} &
\multicolumn{1}{c}{5} &
\multicolumn{1}{c}{13} &
\multicolumn{1}{c}{3} &
\multicolumn{1}{c}{11} &
\multicolumn{1}{c}{7} &
\multicolumn{1}{c}{15} &
\multicolumn{1}{c}{~Group No.} \\
\end{tabular}
\caption{Example of the virtual coordinate system for $k=5$ and $\ell=15$.
The bit value $0$ means that this value is the same as the value of $x$ at this position,
$1$ means this value negated. Group numbers are given for subsequent calls of $k$-ary unbiased operators
on strings $x, y_1, y_2, y_3, y_4$. A group number is formed by reading the column top to bottom 
from $y_4$ to $y_1$ as a binary number. The example for $11 = 1011_2$ is highlighted.
}\label{example-vcs}
\end{figure}

Rounding up, in the case of odd size of a group, ensures\footnote{%
Strictly speaking, rounding up is necessary only in the group 1, since this ensures that this group will have the size one already when querying $y_{k-1}$.
This means that the group 1 will be empty if a $(k+1)$-ary operator is called on $x, y_0, y_1, \ldots, y_{k-1}$, from which it follows that
the group 0 on $x, y_1, \ldots, y_{k-1}$ consists only of ``unrelated'' bits.
In all other groups rounding can be arbitrary: either up or down.} that in no such group bits in $x, y_1, \ldots, y_{k-1}$ coincide (note the absence of $y_0$).
This means that, in the subsequent queries, we can exclude $y_0$ from the arguments, which means that we can still use operators of arity $k$
to query arbitrary strings.

Now we establish a Hamming-preserving bijection between arbitrary strings of length $\ell$ and $k$-ary unbiased operators called on $x, y_1, \ldots, y_{k-1}$.
There are exactly $\ell$ groups among $1 \le j < 2^{k-1}$ for which $n_j = 1$. Based on this, we define an arbitrary injection $g: \{1, \ldots, \ell\} \to \{1, \ldots, 2^{k-1}-1\}$
such that $n_{g_i} = 1$ for all $1 \le i \le \ell$. Since $g$ is an injection, $h = g^{-1}$ exists. We additionally define $h(i) = 0$ in the case $n_i = 0$.
An arbitrary bit string $r = r_1 \ldots r_{\ell}$ of length $\ell$ is mapped to the following $k$-ary unbiased operator:
\begin{equation*}
\langle n_0, n_1, \ldots, n_i, \ldots, n_{2^{k-1}-1} \rangle \to
\langle 0, r_{h(1)}, \ldots, r_{h(i)}, \ldots, r_{h(2^{k-1}-1)} \rangle,
\end{equation*}
where we additionally define $r_0 = 0$ for convenience. On the other hand, when we are given an operator of this sort,
we can remove the identity zero positions at index 0 and at indices where $h(i) = 0$ and unambiguously restore the original string $r$
using the injection $g$ and the remaining positions. This means that this mapping is actually a bijection.
Finally, it is Hamming-preserving, since it can be achieved by applying a permutation
(which is a composition of $g$ and the permutation of the chosen $\ell$ indices which appeared as a result of querying strings $y_i$)
and then an $(x \oplus \cdot)$ operation. The algorithm to query arbitrary strings through this bijection is outlined in Algorithm~\ref{algo-query}.

\begin{algorithm}[!t]
\caption{The way to query arbitrary strings}\label{algo-query}
\begin{algorithmic}
\Function{QueryVirtual}{$x, y_1, \ldots, y_{k-1}; w$}
    \State{-- $x, y_1, \ldots, y_{k-1}$ are the virtual coordinate system}
    \State{-- $w$ is a bit string of length $\ell$}
    \State{$\textsc{Op} := \langle n_0, n_1, \ldots\rangle \to \langle 0, \Call{MakeOperator}{n_1, \ldots; w} \rangle$}
    \State{$z \gets \Call{Op}{x, y_1, \ldots, y_{k-1}}$; $\Call{Query}{z}$}
    \State{\Return $z$}
\EndFunction
\Function{MakeOperator}{$n_1, \ldots, n_z; w$}
    \State{$d \gets []$; $t \gets 1$}
    \For{$i \in [1; z]$}
        \If{$n_i = 1$}
            \State{$d \gets \Call{Append}{d, w_t}$; $t \gets t + 1$}
        \Else
            \State{$d \gets \Call{Append}{d, 0}$}
        \EndIf
    \EndFor
    \State{\Return $d$}
\EndFunction
\end{algorithmic}
\end{algorithm}

A $k$-ary unbiased algorithm cannot access or manipulate the particular bits of the bit strings from the optimization 
domain, neither can it access or set their particular values. 
However, it is free to apply arbitrary $k$-ary unbiased operators. Through the prism
of the just introduced bijection, our algorithm is also able to manipulate individual bits by deciding which exactly
$k$-ary unbiased operators to use. At the first sight, this looks like a violation of the definition of a $k$-ary
unbiased algorithm, however, all these operations are still all performed using only $k$-ary unbiased operators
without any access to representations of individuals. In particular, the following statements hold.

\begin{itemize}
    \item Our algorithm never accesses a bit at the specified index, since the identity of a particular bit is 
          determined by the $k$ bit strings created earlier by unbiased operators, so the bit in question 
          is never fixed in advance and can be any particular bit in different runs of the algorithm.
    \item Our algorithm also does not know whether the bit value is 1 or 0. In each query made through the
          bijection described above, the algorithm defines
          whether this bit (whose position, as said above, is also defined by where exactly and how exactly the $k$ 
          arguments of the corresponding $k$-ary unbiased operator differ from each other) will have the same or 
          the opposite value as in a certain existing query $x$.
\end{itemize}

\section{The Proposed Algorithm}\label{algorithms}

In this section we outline the proposed $k$-ary unbiased algorithm for solving \textsc{OneMax} in $(2 + o(1)) \cdot n / (k - 1) = O(n / k)$ queries.
We explain the general scheme first (Section~\ref{algorithm-general}), then give a number of ideas for how to further speed up the algorithm
(Section~\ref{algorithm-speedups}), and finally give somewhat faster algorithms specialized for $k=3$ and $4$ (Section~\ref{algorithm-3}).

In the text below, as well as in algorithm listings, we make a difference between the query procedure \textsc{Query} and the call to fitness function
\textsc{Fitness}. We always assume that \textsc{Query} must immediately follow the creation of the individual, which associates the fitness value with this individual.
The subsequent calls to \textsc{Fitness} return this computed value and do not count towards the number of queries.
We also assume that whenever \textsc{Query} is called on a bit string which is the optimum, the algorithm immediately terminates, so we do not have to check this condition.

\subsection{The General Form}\label{algorithm-general}

Similarly to the algorithm from~\cite{doerr2012reducing}, we maintain two bit strings, $x$ and $y$, to encode which bits are already guessed right.
The only difference is that we change the meaning to the opposite one: the bits that differ between $x$ and $y$ are guessed right in $x$.
Initially, $x$ is generated uniformly at random and $y = x$.

We optimize bits in blocks of size at most $\ell = 2^{k-1}-1$, where all blocks except for possibly the last one have the maximum possible size $\ell$.
For every block, we first determine which bits to optimize by querying $y_0 = \textsc{FlipEllSame}(x, y)$,
where $\textsc{FlipEllSame} = \langle n_0, n_1 \rangle \to \langle \min(\ell, n_0), 0 \rangle$ is a binary unbiased operator.
Then we build the virtual coordinate system as explained in Section~\ref{system} and Algorithm~\ref{algo-vcs}: 
\begin{equation*}
(y_1, \ldots, y_{k-1}) = \textsc{BuildCoordinates}(x, y_0; k).
\end{equation*}
After that we use this coordinate system to forward queries of an unrestricted algorithm solving \textsc{OneMax} of length $\ell$
to the $\ell$ bits of the original \textsc{OneMax} problem of length $n$ using only $k$-ary unbiased operators.

More formally, we generate random bit strings $w_i$, $i \ge 1$, of length $\ell$ and compute their fitness values $f_i$ as follows:
\begin{equation*}
    f_i = \textsc{Fitness}(\textsc{QueryVirtual}(x, y_1, \ldots, y_{k-1}; w_i)) - \Delta,
\end{equation*}
where $\Delta = (\textsc{Fitness}(x) + \textsc{Fitness}(y_0) - \ell) / 2$ is the term that captures how many ones
the bits unrelated to the $\ell$ chosen bits contribute to the raw fitness value.
We repeat this until, for some $i = i_0$, there remains only one string $w_{\text{opt}}$ of length $\ell$ which,
according to $(w_1, f_1), (w_2, f_2), \ldots, (w_{i_0}, f_{i_0})$, can be the optimum.
Note that this is essentially the same random sampling technique used to solve \textsc{OneMax} in an unrestricted manner,
however, this particular version does not perform restarts and terminates once the optimum is unambiguously determined.

We should explicitly note here that the random bit strings $w_i$ are not the bit strings in the search space of the 
original \textsc{OneMax} problem, but rather the descriptions of the $k$-ary unbiased operators.

Yet another concern 
is that the $w_i$ strings need to be stored somewhere between the queries, which is not explicitly allowed by a 
definition of a $k$-ary unbiased black-box algorithm. However, since any randomized algorithm can be thought of as a 
deterministic algorithm which uses a long enough precomputed sequence of random numbers as a source of its randomized
decisions, the choice of these strings is in fact deterministic, and all these strings can be recomputed from 
scratch at any convenient moment from the number of already performed queries only. It is also possible to prove a more 
general version of this idea to allow an arbitrarily large auxiliary storage, accessible in an arbitrary way, available to 
any $k$-ary unbiased algorithm, while the only information used to populate this storage is the list of fitness 
values.

When the optimum string $w_{\text{opt}}$ is found, we use the same mechanism to set the $\ell$ bits to their 
optimal values while other bits preserve their values. This will be the new value for $x$, and we update $y$ accordingly:
\begin{align*}
    x_{\text{new}} &= \textsc{QueryVirtual}(x, y_1, \ldots, y_{k-1}; w_{\text{opt}}), \\
    y_{\text{new}} &= \textsc{Xor3}(y, x, x_{\text{new}}),
\end{align*}
where $\textsc{Xor3} = \langle n_0, n_1, n_2, n_3 \rangle \to \langle 0, n_1, n_2, 0 \rangle$
is a ternary unbiased operator essentially performing an exclusive-or operation over its three arguments,
also known in the literature as \texttt{selectBits}~\cite{doerr-unbiased-jump-arxiv}.

We outline the entire algorithm as Algorithm~\ref{algo-generic}. For every block of length $\ell$,
it spends one query to determine which $\ell$ bits to optimize, $k-1$ queries to build the virtual
coordinate system, $(2+o(1)) \cdot \ell / \log_2 \ell$ queries in expectation to find the optimum and
two queries to set the next $x$ and $y$. Since $k+2 = O(\log_2\ell) = o(\ell / \log_2 \ell)$,
the entire expected number of queries made by the algorithm is $(2+o(1)) \cdot n / \log_2 \ell = 
(2+o(1)) \cdot n / (k - 1)$, where $o(1)$ is taken relative to $k$.

We note that the functions \textsc{CountConsistent} and \textsc{GetConsistent} in Algorithm~\ref{algo-generic},
operate not on the bit strings in the domain of the original problem, 
but on the bit strings of length $\ell$ which define the $k$-ary unbiased operators (which are stored in $W$)
and on their effective fitness values (which are stored in $F$). Since their first arguments are lists not of the 
original bit strings, but rather of the parameters of the $k$-ary unbiased operators, these functions are free to 
manipulate with their arguments in an arbitrary, unrestricted way.

\begin{algorithm}[!t]
\caption{The generic $k$-ary unbiased algorithm for \texorpdfstring{\textsc{OneMax}}{OneMax}}\label{algo-generic}
\begin{algorithmic}
\Procedure{SolveOneMax}{$n, k$}
    \State{$x \gets \Call{UniformRandomBitStringOfLength}{n}$}
    \State{$\Call{Query}{x}$}
    \State{$y \gets x$}\Comment{Bits different in $x$ and $y$ are guessed right}
    \State{$b \gets 0$}\Comment{The number of bits guessed right}
    \While{true}
        \If{$b = 1$}
            \State{$\textsc{FlipOneWhereSame} := \langle n_0, n_1 \rangle \to \langle 1, 0 \rangle$}
            \State{$\Call{FlipOneWhereSame}{x, y}$}\Comment{Will terminate there}
        \Else
            \State{$\ell \gets \min(b, 2^{k-1}-1)$}
            \State{$k \gets 1 + \lceil \log_2 (\ell + 1) \rceil$} \Comment{Can change at the end}
            \State{$\textsc{FlipEllSame} := \langle n_0, n_1 \rangle \to \langle \ell, 0 \rangle$}
            \State{$y_0 \gets \Call{FlipEllSame}{x, y}$}
            \State{$(y_1, \ldots, y_{k-1}) \gets \Call{BuildCoordinates}{x, y_0; k}$}
            \State{$\Delta \gets (\Call{Fitness}{x} + \Call{Fitness}{y_0} - \ell) / 2$}
            \State{$W = []; F = []$}
            \While{$\Call{CountConsistent}{W, F} > 1$}
                \State{$w \gets \Call{UniformRandomBitStringOfLength}{\ell}$}
                \State{$q \gets \Call{QueryVirtual}{x, y_1, \ldots, y_{k-1}; w}$}
                \State{$f \gets \Call{Fitness}{q} - \Delta$}
                \State{$W \gets \Call{Append}{W, w}$; $F \gets \Call{Append}{F, f}$}
            \EndWhile
            \State{$w_{\text{opt}} \gets \Call{GetConsistent}{W, F}$}
            \State{$x_{\text{new}} \gets \Call{QueryVirtual}{x, y_1, \ldots, y_{k-1}; w_{\text{opt}}}$}
            \State{$\textsc{Xor3} := \langle n_0, n_1, n_2, n_3 \rangle \to \langle 0, n_1, n_2, 0 \rangle$}
            \State{$y \gets \Call{Xor3}{y, x, x_{\text{new}}}$}
            \State{$x \gets x_{\text{new}}$; $b \gets b - \ell$}
        \EndIf
    \EndWhile
\EndProcedure
\end{algorithmic}
\end{algorithm}

\subsection{Performance Improvements}\label{algorithm-speedups}

There are two main sources for performance improvement of the general algorithm.
The first of them is especially significant when $k$ is relatively small. While this issue seems to be
not so important when talking about the unbiased black-box complexity questions ``in the large'',
we shall note, however, that the majority of evolutionary algorithms used in practice feature small arities of their
operators. For instance, differential evolution~\cite{differential-evolution} typically samples
one individual as a linear combination of three different individuals, however it does not sample this individual
directly but crosses it over with another individual and only then queries its fitness. This constitutes an
operator of arity $k = 4$. Finding efficient unbiased algorithms even with small constant arities
may still shed some light on how powerful such operators are compared to mutations and crossovers.

What we note first is that we can use the fitness values of individuals $y_1, \ldots, y_{k-1}$, which constitute
the virtual coordinate system, to seed the unrestricted algorithm for solving \textsc{OneMax}. Although these
individuals are not sampled independently, they still limit the number of consistent optima.
This effect is expected to, and indeed does, reduce the number of subsequent queries for small constant $k$.
Even one such individual reduces the number of potential optima from $2^{\ell}$ to $O(2^{\ell} / \sqrt{\ell})$,
and more individuals perform an even better reduction.
In fact, for very small $k$ even the chance of hitting the optimum inside Algorithm~\ref{algo-vcs},
or immediately after it, is quite noticeable, so in a practical implementation we should check whether
a currently sampled string corresponds to the optimum in $\ell$ bits.

Our second potential improvement comes from an unexpected side,
and its impact is experimentally measured but not yet proven for arbitrary $n$ and $k$.
This is a modification in the random sampling unrestricted algorithm to solve \textsc{OneMax},
which samples a new bit string to test not uniformly at random among all possible bit strings,
but uniformly at random among all \emph{potential optima}, that is, among strings consistent
with the previous samples and measurements.

From now on, we call the original algorithm the \emph{pure}, and this modification the \emph{hack}.
Currently we have no proof for the runtime of the \emph{hack}, hence the name.
In particular, we do not known whether it is $O(n / \log n)$ and whether it is better or worse than
the \emph{pure} algorithm. However, experimental measurements for $2 \le n \le 32$ suggest that
the \emph{hack} performs generally better than the \emph{pure} algorithm for these values
(see Figure~\ref{exp-unrestricted}). In subsequent experiments with the unbiased algorithms,
we will evaluate both versions.

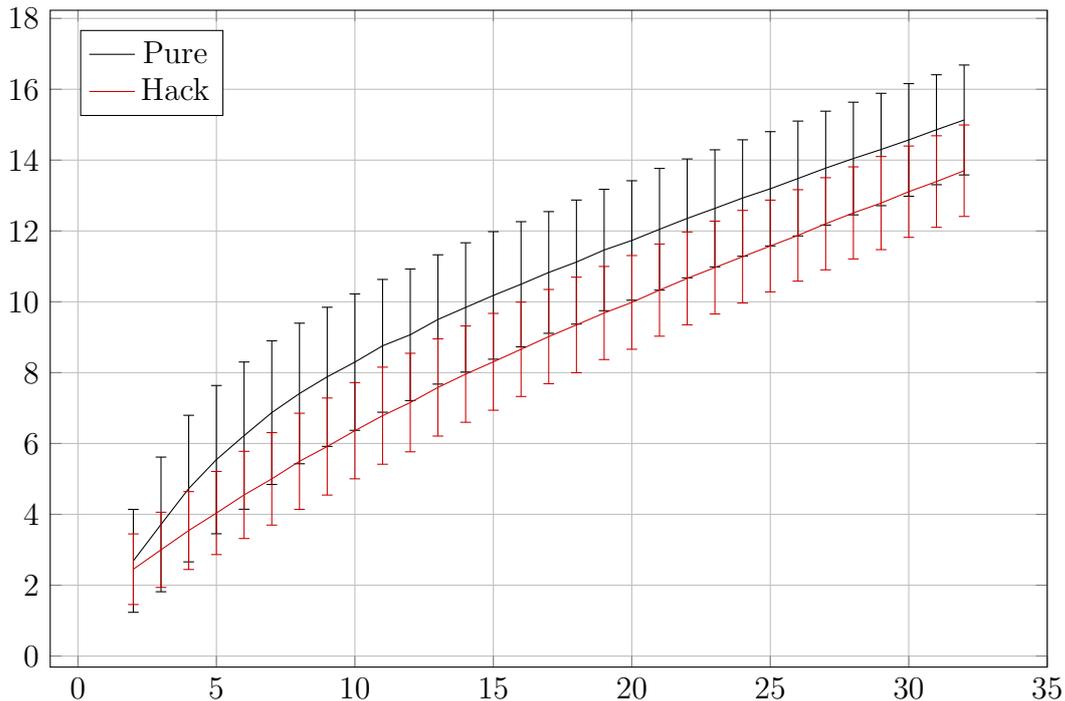
\begin{figure}
\begin{tikzpicture}
\begin{axis}[
    width=\columnwidth,
    height=0.7\columnwidth,
    cycle list name=myplotcycle,
    legend pos=north west,
    grid=both
]
\addplot plot[error bars/.cd, y dir=both, y explicit] table[y error=dev] {\UnrestrictedPure};
\addlegendentry{Pure};
\addplot plot[error bars/.cd, y dir=both, y explicit] table[y error=dev] {\UnrestrictedHack};
\addlegendentry{Hack};
\end{axis}
\end{tikzpicture}
\caption{Experiment results for unrestricted algorithms.
$10^4$ independent runs were performed for each $n$.
Means and standard deviations are plotted.
}\label{exp-unrestricted}
\end{figure}

\subsection{Algorithms for $k = 3$ and 4}\label{algorithm-3}

In this section, we describe a custom implementation of the block-wise optimization of \textsc{OneMax}
using ternary unbiased operators ($k = 3$). However, since the block size $\ell = 3$,
we simply perform the fitness-dependent case analysis instead of explicit construction of the virtual coordinate
system and subsequent random sampling.

The algorithm is presented in Algorithm~\ref{algo-3}, and the operators it uses are given in Algorithm~\ref{algo-3-ops}.
By following its branches and computing the probabilities of their execution,
one can prove the expected running time of $9n / 8 \pm O(1)$. Note that this is faster than the best known binary unbiased algorithm,
whose expected running time is $2n \pm O(1)$~\cite{doerr-johannsen-faster-blackbox}. In a similar way, we designed an algorithm with $k = 4$ and $\ell = 7$,
whose expected running time is $765/896n \pm O(1)$.

\begin{algorithm}[!t]
\caption{The operators for Algorithm~\ref{algo-3}}\label{algo-3-ops}
\begin{algorithmic}
    \State{$\textsc{SameX} := \langle n_0, n_1 \rangle \to \langle x, 0 \rangle$}
    \State{$\textsc{OneWhereEqual} := \langle n_0, n_1, n_2, n_3 \rangle \to \langle 1, 0, 0, 0 \rangle$}
    \State{$\textsc{OneWhere2ndDiffers} := \langle n_0, n_1, n_2, n_3 \rangle \to \langle 0, 1, 0, 0 \rangle$}
    \State{$\textsc{OneWhere3rdDiffers} := \langle n_0, n_1, n_2, n_3 \rangle \to \langle 0, 0, 1, 0 \rangle$}
    \State{$\textsc{TwoWhere3rdDiffers} := \langle n_0, n_1, n_2, n_3 \rangle \to \langle 0, 0, 2, 0 \rangle$}
    \State{$\textsc{Complicated} := \langle n_0, n_1, n_2, n_3 \rangle \to \langle 0, n_1, 0, 1 \rangle$}
    \State{$\textsc{Xor3} := \langle n_0, n_1, n_2, n_3 \rangle \to \langle 0, n_1, n_2, 0 \rangle$}
\end{algorithmic}
\end{algorithm}

\begin{algorithm}[!t]
\caption{The custom unbiased algorithm for $k = 3$}\label{algo-3}
\begin{algorithmic}
\scriptsize
\Procedure{SolveOneMax3}{$n$}
    \State{$x \gets \Call{UniformRandomBitStringOfLength}{n}$; $y \gets x$}
    \State{$b \gets n$}\Comment{How many bits left to be optimized}
    \While{true}
        \If{$b = 1$}
            \State{$o \gets \Call{Same1}{x, y}$; $\Call{Query}{o}$} \Comment{Optimum}
        \ElsIf{$b = 2$}
            \If{$\Call{Fitness}{x} = n - 2$}
                \State{$o \gets \Call{Same2}{x, y}$; $\Call{Query}{o}$} \Comment{Optimum}
            \Else
                \State{$z \gets \Call{Same1}{x, y}$; $\Call{Query}{z}$}
                \If{$\Call{Fitness}{z} = n - 2$}
                    \State{$o \gets \Call{OneWhereEqual}{x, y, z}$; $\Call{Query}{o}$}
                \EndIf \Comment{Either $z$ or $o$ was the optimum}
            \EndIf
        \Else
            \State{$m \gets \Call{Same3}{x, y}$; $\Call{Query}{m}$}
            \If{$\Call{Fitness}{m} = \Call{Fitness}{x} + 3$}
                \State{$x \gets m$} \Comment{$y$ is already OK}
            \ElsIf{$\Call{Fitness}{m} = \Call{Fitness}{x} - 3$}
                \State{$y \gets \Call{Xor3}{x, y, m}$; $\Call{Query}{y}$} \Comment{$x$ is already OK}
            \ElsIf{$\Call{Fitness}{m} = \Call{Fitness}{x} + 1$}
                \State{$p \gets \Call{TwoWhere3rdDiffers}{x, y, m}$; $\Call{Query}{p}$}
                \If{$\Call{Fitness}{p} = \Call{Fitness}{x} + 2$}
                    \State{$x \gets p$}
                \Else
                    \State{$r \gets \Call{Complicated}{x, m, p}$; $\Call{Query}{r}$}
                    \If{$\Call{Fitness}{r} = \Call{Fitness}{x} + 2$}
                        \State{$x \gets r$}
                    \Else
                        \State{$x \gets \Call{Xor3}{x, p, r}$; $\Call{Query}{x}$}
                    \EndIf
                \EndIf
                \State{$y \gets \Call{Xor3}{x, y, m}$; $\Call{Query}{y}$}
            \Else
                \State{$p \gets \Call{OneWhere3rdDiffers}{x, y, m}$; $\Call{Query}{p}$}
                \If{$\Call{Fitness}{p} = \Call{Fitness}{x} + 1$}
                    \State{$x \gets p$}
                \Else
                    \State{$r \gets \Call{OneWhere2ndDiffers}{x, m, p}$}
                    \State{$\Call{Query}{r}$}
                    \If{$\Call{Fitness}{r} = \Call{Fitness}{x} + 1$}
                        \State{$x \gets r$}
                    \Else
                        \State{$x \gets \Call{Xor3}{m, p, r}$; $\Call{Query}{x}$}
                    \EndIf
                \EndIf
                \State{$y \gets \Call{Xor3}{x, y, m}$; $\Call{Query}{y}$}
            \EndIf
            \State{$b \gets b - 3$}
        \EndIf
    \EndWhile
\EndProcedure
\end{algorithmic}
\end{algorithm}

\section{Experiments}\label{experiments}

We have implemented a small programming platform for experimenting with unbiased algorithms of fixed arity.
In this platform, the bit strings of individuals are encapsulated in a so-called ``unbiased processor'',
and the individuals are accessed through handles, which reveal only their fitness. Unbiased operators
are naturally encoded as functions which map one integer array, $\langle n_0, n_1, \ldots \rangle$,
to another one, $\langle d_0, d_1, \ldots \rangle$. With these precautions, the platform ensures that one implements
only unbiased algorithms. The source code is available on GitHub.\footnote{\url{https://github.com/mbuzdalov/unbiased-bbc}}

\begin{figure}
\scalebox{0.75}{
\begin{tikzpicture}
\begin{axis}[
    width=1.3333\textwidth,
    height=0.53333\textheight,
    cycle list name=myplotcycle,
    legend pos=north west,
    grid=both
]
\addplot plot[error bars/.cd, y dir=both, y explicit] table[y error=dev] {\ArityTwo};
\addlegendentry{standard, $k=2$};
\addplot plot[error bars/.cd, y dir=both, y explicit] table[y error=dev] {\GenericArityThreePure};
\addlegendentry{generic, $k=3$, pure};
\addplot plot[error bars/.cd, y dir=both, y explicit] table[y error=dev] {\GenericArityThreeHack};
\addlegendentry{generic, $k=3$, hack};
\addplot plot[error bars/.cd, y dir=both, y explicit] table[y error=dev] {\CustomArityThree};
\addlegendentry{custom, $k=3$};
\addplot plot[error bars/.cd, y dir=both, y explicit] table[y error=dev] {\GenericArityFourPure};
\addlegendentry{generic, $k=4$, pure};
\addplot plot[error bars/.cd, y dir=both, y explicit] table[y error=dev] {\GenericArityFourHack};
\addlegendentry{generic, $k=4$, hack};
\addplot plot[error bars/.cd, y dir=both, y explicit] table[y error=dev] {\CustomArityFour};
\addlegendentry{custom, $k=4$};
\addplot plot[error bars/.cd, y dir=both, y explicit] table[y error=dev] {\GenericArityFivePure};
\addlegendentry{generic, $k=5$, pure};
\addplot plot[error bars/.cd, y dir=both, y explicit] table[y error=dev] {\GenericArityFiveHack};
\addlegendentry{generic, $k=5$, hack};
\addplot plot[error bars/.cd, y dir=both, y explicit] table[y error=dev] {\GenericAritySixPure};
\addlegendentry{generic, $k=6$, pure};
\addplot plot[error bars/.cd, y dir=both, y explicit] table[y error=dev] {\GenericAritySixHack};
\addlegendentry{generic, $k=6$, hack};
\end{axis}
\end{tikzpicture}}
\caption{Experiment results for $k$-ary unbiased algorithms on \textsc{OneMax}. 100 independent runs were made for each point.}\label{exp-unbiased}
\end{figure}
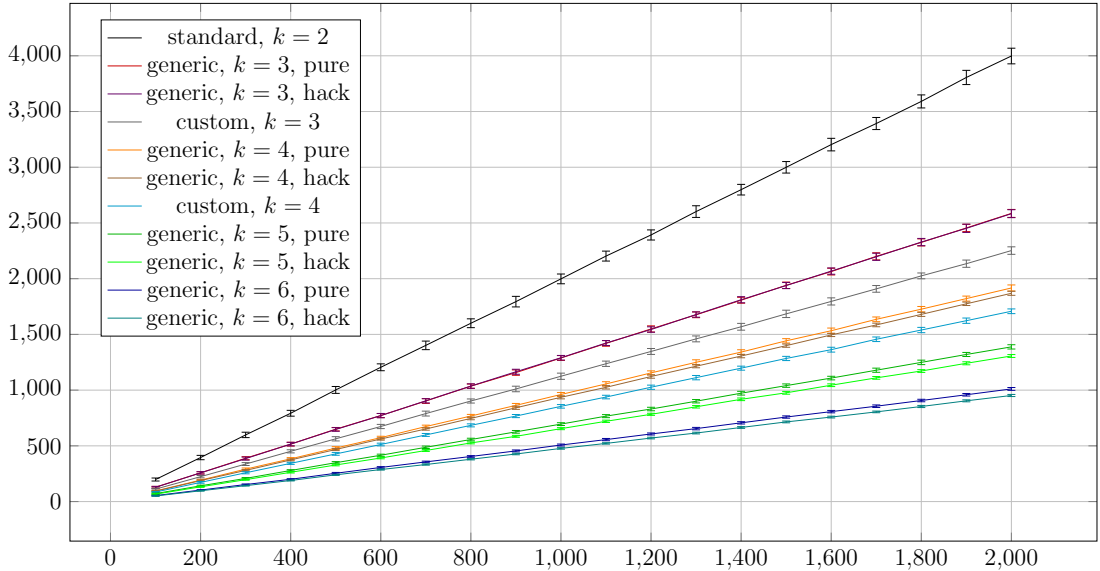

By the means of this platform, we have implemented the well-known binary algorithm, Algorithm~\ref{algo-3} for $k=3$ and 4,
as well as the \emph{pure} and \emph{hack} flavors of the generic algorithm for arbitrary fixed arities.
We could handle only arities $3 \le k \le 6$, since with $k = 7$ the block size becomes $\ell = 63$,
the size which currently cannot be solved by neither \emph{pure} nor \emph{hack} random sampling in reasonable time.
The experimental results are presented in Figure~\ref{exp-unbiased}.

For $k = 3$, both generic algorithms showed the performance of $1.29n$, which is somewhat greater than $9n/8 = 1.125n$
by the custom algorithm. However, for $k=4$ the runtimes of $\approx 0.958n$ was shown by the pure algorithm,
and $\approx 0.934n$ by the hack, both smaller than $n$.
The custom algorithm for $k=4$ is still slightly faster with $765/896n \approx 0.854n$.
For $k = 5$ the pure and hack versions were
$\approx 0.694n$ and $\approx 0.653n$ correspondingly. Finally, for $k = 6$ the results were close to half the size,
with the pure algorithm being slightly above, $\approx 0.505n$, while the hack is slightly below, $\approx 0.476n$.

\section{Conclusion}\label{conclusion}

We presented a cleaner and more efficient way to prove that the $k$-ary unbiased black-box complexity of \textsc{OneMax}
is $O(n / k)$. Our approach enabled to define the explicit constants: the running time is shown to be $(2 + o(1)) \cdot n / (k - 1)$,
where $o(1)$ relates to $k$. We also showed that unbiased operators are powerful enough
to enable solving exponentially large parts of the original problem by arbitrary unrestricted algorithms.
Finally, this approach is efficient enough even for the smallest arities $k \ge 3$.
We hope that this paper will accelerate and simplify the research on higher-arity operators and their impact
on the performance of randomized search heuristics.

\section{Acknowledgments}

This work was supported by Russian Science Foundation under the agreement No.~17-71-20178.

\bibliographystyle{abbrv}
\bibliography{../../../../bibliography}

\end{document}